\documentclass{article} 
\usepackage{amsmath}
\usepackage{amssymb}
\usepackage{graphicx}
\usepackage{mathtools}
\usepackage{amsfonts}
\usepackage{amsthm,bm}
\usepackage{color}
\usepackage{natbib}

\usepackage[colorlinks=true,bookmarks=false,linkcolor=blue,urlcolor=blue,citecolor=blue,breaklinks=true]{hyperref}

 \usepackage{pifont}

\newtheorem{theorem}{Theorem}[section]

\newtheorem{proposition}{Proposition}[section]
\newtheorem{lemma}{Lemma}[section]

\newtheorem{assumption}{Assumption}

\graphicspath {{figures/}}

\usepackage{caption}
\usepackage[super]{nth}
\usepackage{enumitem}
\usepackage{subcaption}
\usepackage{algorithm}
\usepackage{algorithmic}
\usepackage{caption}
\usepackage{authblk}

\usepackage[top=1in, bottom=1.25in, left=1.25in, right=1.25in]{geometry}

\usepackage{letltxmacro}
\newlength{\commentindent}
\makeatletter
\renewcommand{\algorithmiccomment}[1]{\unskip\hfill\makebox[\commentindent][l]{//~#1}\par}
\LetLtxMacro{\oldalgorithmic}{\algorithmic}
\renewcommand{\algorithmic}[1][0]{%
  \oldalgorithmic[#1]%
  \renewcommand{\ALC@com}[1]{%
    \ifnum\pdfstrcmp{##1}{default}=0\else\algorithmiccomment{##1}\fi}%
}
\makeatother

\newcounter{procedure}
\makeatletter

\title{Dual Averaging is Surprisingly Effective for Deep Learning Optimization}
\author[a b]{Samy Jelassi \footnote{Work done while interning at Facebook AI Research NYC. }}
\author[c]{Aaron Defazio}
\affil[a]{Princeton University, Princeton}
\affil[b]{Courant Institute of Mathematical Sciences, New York University, New York}
\affil[c]{Facebook AI Research, New York} 
\date{\today}

\begin{document}

\maketitle

\begin{abstract}

First-order stochastic optimization methods are currently the most widely used class of methods for training deep neural networks. However, the choice of the optimizer has become an ad-hoc rule that can significantly affect the performance. For instance, SGD with momentum (SGD+M) is typically used in computer vision (CV) and Adam is used for training transformer models for Natural Language Processing (NLP). Using the wrong method can lead to significant performance degradation. Inspired by the dual averaging algorithm, we propose Modernized Dual Averaging (MDA), an optimizer that is able to perform as well as SGD+M in CV and as Adam in NLP. Our method is not adaptive and is significantly simpler than Adam. We show that MDA induces a decaying uncentered $L_2$-regularization compared to vanilla SGD+M and hypothesize that this may explain why it works on NLP problems where SGD+M fails. 
\end{abstract}


\section{Introduction}
Stochastic first-order optimization methods have been extensively employed for training neural networks. It has been empirically observed that the choice of the optimization algorithm is crucial for obtaining a good accuracy score. For instance, stochastic variance-reduced methods perform poorly in computer vision (CV) \citep{defazio2019ineffectiveness}. On the other hand, SGD with momentum (SGD+M) \citep{bottou1991stochastic,lecun1998gradient,bottou2008tradeoffs} works particularly well on CV tasks and Adam \citep{kingma2014adam} out-performs other methods on natural language processing (NLP) tasks \citep{choi2019empirical}.  In general, the choice of optimizer, as well as its hyper-parameters, must be included among the set of hyper-parameters that are searched over when tuning.



In this work we propose \textit{Modernized Dual Averaging} (MDA), an optimizer that matches the performance of SGD+M on CV tasks and Adam on NLP tasks, providing the best result in both worlds. Dual averaging \citep{nesterov2009primal} and its variants have been heavily explored in the convex optimization setting. Our \textit{modernized} version updates dual averaging with a number of changes that make it effective for non-convex problems.
Compared to other methods, dual averaging has the advantage of accumulating new gradients with \textit{non-vanishing} weights. Moreover, it has been very successful for regularized learning problems due to its ability to obtain desirable properties (e.g. a sparse solution in Lasso) faster than SGD \citep{xiao2010dual}. 

In this paper, we point out another advantage of dual averaging compared to SGD.  As we show in \autoref{sec:da}, under the right parametrization, dual averaging is equivalent to SGD applied to the same objective function but with a \textit{decaying $\ell_2$-regularization}.
This induced $\ell_2$-regularization has two primary implications for neural network training. Firstly, from an optimization viewpoint, regularization smooths the optimization landscape, aiding optimization.
From a learning viewpoint, $\ell_2$-regularization (often referenced as weight decay) is crucial for generalization performance \citep{krogh1992simple,bos1996using,wei2019regularization}.  Through an empirical investigation, we demonstrate that this implicit regularization effect is beneficial as MDA outperforms SGD+M in settings where the latter perform poorly.

\subsection*{Contributions}
This paper introduces MDA, an algorithm that matches the performance of the best first-order methods in a wide range of settings. More precisely, our contributions can be divided as follows:
\begin{itemize}
\item[--] \textit{Adapting dual averaging to neural network training}: We build on the subgradient method with double averaging \citep{nesterov2015quasi} and adapt it to deep learning optimization. In particular, we specialize the method to the $L_2$-metric, modify the hyper-parameters and design a proper scheduling of the parameters.
\item[--] \textit{Theoretical analysis in the non-convex setting}: Leveraging a connection between SGD and dual averaging, we derive a convergence analysis for MDA in the non-convex and smooth optimization setting. This analysis is the first convergence proof for a dual averaging algorithm in the non-convex case.  
\item[--] \textit{MDA matches the performance of the best first-order methods}: We investigate the effectiveness of dual averaging in CV and NLP tasks. For supervised classification, we match the test accuracy of SGD+M on  CIFAR-10 and ImageNet. For image-to-image tasks, we match the performance of Adam on MRI reconstruction on the fastMRI challenge problem. For NLP tasks, we match the performance of Adam on machine translation on IWSLT'14 De-En, language modeling on Wikitext-103 and masked language modeling on the concatenation of BookCorpus and English Wikipedia.
\end{itemize}




\subsection*{Related Work}

\paragraph{First-order methods in deep learning.} While SGD  and Adam are the most popular methods, a wide variety of optimization algorithms have been applied to the training of neural networks. Variants of SGD such as momentum methods and Nesterov's accelerated gradient improve the training performance \citep{sutskever2013importance}. Adaptive methods as Adagrad \citep{duchi2011adaptive}, RMSprop \citep{hinton2012neural} and Adam have been shown to find solutions that generalize worse than those found by non-adaptive methods on several state-of-the-art deep learning models \citep{wilson2017marginal}. \cite{berrada2018deep} adapted the Frank-Wolfe algorithm to design an optimization method that offers good generalization performance while requiring minimal hyper-parameter tuning compared to SGD. 

\paragraph{Dual averaging.} Dual averaging is one of the most popular algorithms in convex optimization and presents two main advantages. In regularized learning problems, it is known to more efficiently obtain the desired regularization effects compared to other methods as SGD \citep{xiao2010dual}. Moreover, dual averaging fits the distributed optimization setting \citep{duchi2011dual,tsianos2012push,hosseini2013online,shahrampour2013exponentially,colin2016gossip}. Finally, this method seems to be effective in manifold identification \citep{lee2012manifold,duchi2016asymptotic}. Our approach differs from these works as we study dual averaging in the non-convex optimization setting.

\paragraph{Convergence guarantees in non-convex optimization.} While obtaining a convergence rate for SGD when the objective is smooth is standard (see e.g. \cite{bottou2016optimization}), it is more difficult to analyze other algorithms in this setting.  Recently, \cite{zou2018weighted,ward2019adagrad,li2019convergence} provided rates for the convergence of variants of Adagrad towards a stationary point. \cite{defossez2020convergence} builds on the techniques introduced in \cite{ward2019adagrad} to derive a convergence rate for Adam. In the non-smooth weakly convex setting, \cite{davis2019stochastic} provides a convergence analysis for SGD and \cite{zhang2018convergence} for Stochastic Mirror Descent. Our analysis for dual averaging  builds upon the recent analysis of SGD+M by \citet{defazio2020mom}.

\paragraph{Decaying regularization} Methods that reduce the amount of regularization used over time have been explored in the convex case. \cite{blackboxoptimal2016} show that it's possible to use methods designed for strongly convex optimization to obtain optimal rates for other classes of convex functions by using polynomially decaying regularization with a restarting scheme.  In \cite{smallgradients2018}, it is shown that adding regularization centered around a sequence of points encountered during optimization, rather than the zero vector, results in better convergence in terms of gradient norm on convex problems. 


\section{Modernizing dual averaging}\label{sec:da}


As we are primarily interested in neural network training, we focus in this work on the unconstrained stochastic minimization problem 
\begin{align}\label{eq:prob}
    \min_{x\in \mathbb{R}^n}\mathbb{E}_{\xi \sim P}[f(x,\xi)]:=f(x).
\end{align}

We assume that $\xi$ can be sampled from a fixed but unknown probability distribution $P$. Typically, $f(x,\xi)$ evaluates the loss of the decision rule parameterized by $x$ on a data point $\xi$. Finally, $f\colon \mathbb{R}^n\rightarrow \mathbb{R}$ is a (potentially) non-convex function.
 
 \paragraph{Dual averaging.} To solve \eqref{eq:prob}, we are interested in the dual averaging algorithm \citep{nesterov2009primal}. In general, this scheme is based on a mirror map $\Phi\colon \mathbb{R}^n\rightarrow \mathbb{R}$ assumed to be strongly convex. An exhaustive list of popular mirror maps is present in \cite{bregman1967relaxation,teboulle1992entropic,eckstein1993nonlinear,bauschke1997legendre}. In this paper, we focus on the particular choice 
 \begin{align}\label{eq:mirror_mapchoice}
     \Phi(x):=\frac{1}{2}\|x-x_0\|^2,
 \end{align}
 where $x_0\in \mathbb{R}^n$ is the initial point of the algorithm. 


 \begin{algorithm}
\begin{algorithmic} 
\STATE \textbf{Input}: initial point $x_0$, scaling parameter sequence $\{\beta_k\}_{k=1}^T$ , step-size sequence $\{\lambda_k\}_{k=1}^T$ , stopping time $T$. 
\FOR{$k=0\dots T$}
\STATE Sample $\xi_k \sim P$ and compute stochastic gradient $g_k=\nabla f(x_k,\xi_k)$
\STATE $s_{k}=s_{k-1}+\lambda_k g_k.$ \COMMENT{Update the sum of gradients}
\STATE $x_{k+1} = \mathrm{argmin}_{x\in \mathbb{R}^n} \left\{\langle s_{k},x\rangle +\frac{\beta_k}{2}\|x-x_0\|_2^2\right\}$ \COMMENT{Update the iterate} 
\ENDFOR
\RETURN $\hat{x}_{T} = \frac{1}{T+1} \sum_{k=0}^T  x_k$.
\end{algorithmic}
\caption{\label{alg:DA}(Stochastic) dual averaging}
\end{algorithm}

Dual averaging generates a sequence of iterates $\{x_k,s_k\}_{k=0}^T$ as detailed in  \autoref{alg:DA}. At time step $k$ of the algorithm, the algorithm receives $g_k$ and updates the sum of the weighted gradients $s_k$. Lastly, it updates the next iterate $x_{k+1}$ according to a proximal step. Intuitively, $x_{k+1}$ is chosen to minimize an averaged first-order approximation to the function $f$, while the regularization term $\beta_k\Phi(x)$ prevents the sequence $\{x_k\}_{k=0}^T$ from oscillating too wildly. The sequence $\{\beta_k\}_{k=1}^T$ is chosen to be non-decreasing in order to counter-balance the growing influence of $\langle s_k,x\rangle$. We remark that the update in Algorithm \ref{alg:DA} can be rewritten as: 
\begin{align}\label{eq:DA_update}
    x_{k+1}= -s_k/\beta_k.
\end{align}


In the convex setting, \cite{nesterov2009primal} chooses $\beta_{k+1}=\beta_k + 1/\beta_k$ and $\lambda_k=1$ and shows convergence of the average iterate $\hat{x}_T$ to the optimal solution at a rate of $O(1/\sqrt{T})$. That sequence of $\beta$ values grows proportionally to the square-root of $k$, resulting in a method which an effective step size that decays at a rate $O(1/\sqrt{T})$. This rate is typical of decaying step size sequences used in first order stochastic optimization methods when no strong-convexity is present.

\paragraph{Connection with SGD.} With our choice of mirror map \eqref{eq:mirror_mapchoice}, stochastic mirror descent (SMD) is equivalent to SGD whose update is
\begin{align}\label{eq:update_MD}
    x_{k+1}&=  x_{k}-\eta_k g_k.
\end{align}
Dual averaging and SMD share similarities. While in constrained optimization the two algorithms are different \cite{juditsky2019unifying}, they yield the same update in the unconstrained case when $\lambda_k=\eta_k$ and $\beta_k=1$. In this paper, we propose another viewpoint on the relationship between the two algorithms.
 \begin{proposition}\label{prop:mdda}
  Let $f\colon \mathbb{R}^n\rightarrow \mathbb{R}$ be a function and let $T>0$. Let $\{h^{(k)}\}_{k=0}^{T}$ be a sequence of functions such that $h^{(k)}\colon \mathbb{R}^n\rightarrow \mathbb{R}$ and 
  \begin{align}
      h^{(k)}(x)&=f(x)+\frac{\alpha_k}{2}\|x-x_0\|^2,
  \end{align}
 where $\{\alpha_k\}_{k=0}^{T}$ is a sequence in $\mathbb{R}.$ Then, for $k\in \{1,\dots,T\}$, the update of dual averaging at iteration $k$ for the minimization problem on $f$ is equivalent to the one of SGD for the minimization problem on $h^{(k)}$ when 
 \begin{align}\label{eq:right_params_choice}
     \eta_k=\frac{\lambda_k}{\beta_k} \quad \text{and} \quad \alpha_k = \frac{\beta_{k}-\beta_{k-1}}{\lambda_k}.
 \end{align}
 \end{proposition}
\begin{proof}[Proof of \autoref{prop:mdda}.] We start by deriving the SGD update for $h^{(k)}$. 
\begin{align}\label{eq:MDupdatetimvar}
    x_{k+1}&= x_{k}-\eta_kg_k-\eta_k\alpha_k x_{k}.
\end{align}
We now rewrite the update 
of dual averaging. By evaluating \eqref{eq:DA_update} at iterations $k$ and $k-1$, we obtain:
\begin{align}\label{eq:DAvar}
    \begin{cases}
      x_{k+1}= -s_k/\beta_k\\
        x_{k}= -s_{k-1}/\beta_{k-1}
    \end{cases}\implies x_{k+1}= x_{k}-\frac{\lambda_k}{\beta_k}g_k-\left(1-\frac{\beta_{k-1}}{\beta_k}\right)x_{k}.
\end{align}
By comparing \eqref{eq:MDupdatetimvar} and \eqref{eq:DAvar}, we obtain \eqref{eq:right_params_choice}.
\end{proof}

\autoref{prop:mdda} shows that dual averaging implicitly induces a time-varying $L_2$-regularization to an SGD update. 

\paragraph{Modernized dual averaging.} The \textit{modernized dual averaging} (MDA) algorithm, our adaptation of dual averaging for deep learning optimization, is given in \autoref{alg:DA_momentum}.

\begin{algorithm}
\begin{algorithmic} 
\STATE \textbf{Input}: $x_0\in \mathbb{R}^n$ initial point, $\eta_k>0$ stepsize sequence, $c_k$ momentum parameter sequence, $T>0$ stopping time.
\STATE  Initialize $s_{-1} = 0$.
\FOR{$k=0\dots T-1$}
\STATE Set the scaling coefficient $\beta_k=\sqrt{k+1}$ and the stepsize $\lambda_k=\eta_k\sqrt{k+1}.$ 
\STATE Sample $\xi_k \sim P$ and compute stochastic gradient $g_k=\nabla f(x_k,\xi_k).$\STATE $s_{k}=s_{k-1}+\lambda_k g_k$ \COMMENT{Update the sum of gradients}
\STATE $ z_{k+1} =x_0-s_k/\beta_k$ \COMMENT{Update the dual averaging iterate}
\STATE $x_{k+1}=(1-c_{k+1})x_k+c_{k+1} z_{k+1}$ \COMMENT{Update the averaged iterate}
\ENDFOR
\RETURN $x_{T} .$
\end{algorithmic}
\caption{\label{alg:DA_momentum}Modernized dual averaging (MDA)}
\end{algorithm}


MDA differs from dual averaging in the following fundamental ways. Firstly, it maintains an iterate $x_{k+1}$ obtained as a weighted average of the previous average $x_k$ and the current dual averaging iterate $z_{k+1}.$ It has been recently noticed that this averaging step can be interpreted as introducing momentum in the algorithm \citep{sebbouh2020convergence,tao2018primal} (more details in \autoref{sec:sgdmom_proof}). For this reason, we will refer to $c_k$ as the momentum parameter. While dual averaging with double averaging has already been introduced in \cite{nesterov2015quasi}, it is our choices of parameters that make it suitable for non-convex deep learning objectives. In particular, our choice of $\beta_k$ and $\lambda_k$, motivated by a careful analysis of our theoretical convergence rate bounds,  result in the following adaptive regularization when viewed as regularized SGD with momentum:
\begin{align}\label{eq:params_noncvx}
     \alpha_k = \frac{\sqrt{k+2}-\sqrt{k+1}}{\eta_k \sqrt{k+2}} \approx\frac{1}{k+2}
\end{align}

In practice, a schedule of the momentum parameter (in our case $c_k$) and learning rate ($\eta_k$) must also be chosen to get the best performance out of the method. We found that the schedules used for SGD or Adam can be adapted for MDA with small modifications. For CV tasks, we found it was most effective to use modifications of existing stage-wise schemes where instead of a sudden decrease at the end of each stage, the learning rate decreases linearly to the next stages value, over the course of a few epochs. For NLP problems, a warmup stage is necessary following the same schedule typically used for Adam. Linear decay, rather than inverse-sqrt schedules, were the most effective post-warmup. 

For the momentum parameter, in each case the initial value can be chosen to match the momentum $\beta$ used for other methods, with the mapping $c_{k}=1-\beta$. Our theory suggests that when the learning rate is decreased, $c_{k}$ should be increased proportionally (up to a maximum of 1) so we used this rule in our experiments, however it doesn't make a large difference.

\vspace*{-.3cm}

\section{Convergence analysis}
\label{sec:convergence}Our analysis requires the following assumptions. We assume that $f$ has Lipschitz gradients but in not necessarily convex.
Similarly to \cite{robbins1951stochastic}, we assume unbiasedness of the gradient estimate and boundedness of the variance.
\begin{assumption}[Stochastic gradient oracle] \label{ass:sto_grad}
We make the two following assumptions.\\
(A1) Unbiased oracle: $\mathbb{E}_{\xi \sim P}[\nabla f(x,\xi)]= \nabla f(x)$.\\
(A2) Bounded second moment: $\mathbb{E}_{\xi \sim P}[\|\nabla f(x,\xi)\|_2^2]\leq \sigma^2$.
\end{assumption}
\begin{assumption}[Boundedness of the domain]\label{ass:bounddomain}
Let $x_0 \in \mathbb{R}^n$. Then, we assume that there exists $R>0$ such that $R^2=\sup_{x\in \mathbb{R}^d} \|x-x_0\|_2^2 < \infty.$
\end{assumption}

\begin{theorem}\label{thm:mda}
  Let $f$ be a Lipschitz-smooth function with minimum $f^*$. Let \autoref{ass:sto_grad} and \autoref{ass:bounddomain} for \eqref{eq:prob} hold. Let $x_0 \in \mathbb{R}^n$ be the initial point of MDA. Assume that we run MDA for $T$ iterations. Let $z_1,\dots,z_T$ and $x_1,\dots,x_T$ be the points returned by MDA and set $\lambda_k=\eta_k \sqrt{k+1}$, $\beta_k=\sqrt{k+1}$ and $c_k=c$ where $\eta_k = 1/\sqrt{T}$ and $c\in (0,1].$ Assume that $T\geq L^2/c^2.$ Then, we have:
  \begin{equation}\label{eq:da_rate}
\begin{split}
        &\frac{1}{2T}\sum_{k=0}^T(\|\nabla f(x_k)\|_2^2+\|\nabla f(z_k)\|_2^2)\\
        &\leq \frac{2((f(x_{0})-f^*)-\mathbb{E}[f(z_{T+1})-f^*])}{\sqrt{T}}\\
        &+2\left(\frac{1}{c}-1\right)\frac{(L+1)(f(x_{0})-f^*)-(L+\alpha_T)\mathbb{E}[f(x_{T})-f^*] }{T}\\
    %
    %
    %
    &+ 2 \left[\left(\frac{L}{\sqrt{T}} +\frac{\log(T+1)}{T}\right)\sigma^2+ \left( \frac{L\log(T)}{T}+\frac{2\log(T)}{\sqrt{T}} \right) R^2\right],\\
 \end{split}
\end{equation}
where $f^*$ is the value of $f$ at a stationary point and $\alpha_T = \sqrt{T}\left(1 - \frac{\sqrt{T+1}}{\sqrt{T+2}}\right).$
\end{theorem}
\autoref{thm:mda} informs us that the convergence rate of MDA to a stationary point is of $O(1/\sqrt{T})$ and is similar to the one obtained with SGD. A proof of this statement can be found in \autoref{sec:mda_proof}.

\section{Numerical experiments}\label{sec:num_exp}
We investigate the numerical performance of MDA on a wide
range of learning tasks, including image classification, MRI reconstruction,
neural machine translation (NMT) and language modeling. We performed a comparison against both SGD+M and Adam. Depending on the task, one of these two methods is considered the current state-of-the-art for the problem. 
For each task, to enable a fair comparison, we perform a grid-search over step-sizes, weight decay and learning rate schedule to obtain the best result for each method. 
For our CV experiments we use the torchvision package, and for our NLP experiments we use fairseq \citep{ott2019fairseq}.
We now briefly explain each of the learning tasks and present our results. Details of our experimental setup and experiments on Wikitext-103 can be respectively found in \autoref{sec:experimental-setup} and \autoref{sec:further_exps}.

\subsection{Image classification}
We run our image classification experiments on the CIFAR-10 and ImageNet datasets. The CIFAR-10 dataset consists of 50k training images and 10k testing images. The ILSVRC 2012 ImageNet dataset has 1.2M
training images and 50k validation images. We train a pre-activation ResNet-152 on CIFAR-10 and a ResNet-50 model for ImageNet. Both architectures are commonly used base-lines for these problems.
We follow the settings described in \cite{he2016deep} for training. 
\autoref{fig:cv} (a) represents the accuracy obtained on CIFAR-10. MDA achieves a slightly better accuracy compared to SGD with momentum (by 0.36\%). This is an interesting result as SGD with momentum serves as first-order benchmark on CIFAR-10. We speculate this difference is due to the beneficial properties of the decaying regularization that MDA contains. 
\autoref{fig:cv} (b) represents the accuracy obtained on ImageNet. In this case the difference between MDA and SGD+M is within the standard errors, with Adam trailing far behind.

\begin{figure}[t]
 
    \begin{minipage}{.5\textwidth}
        \includegraphics[width=.95\linewidth]{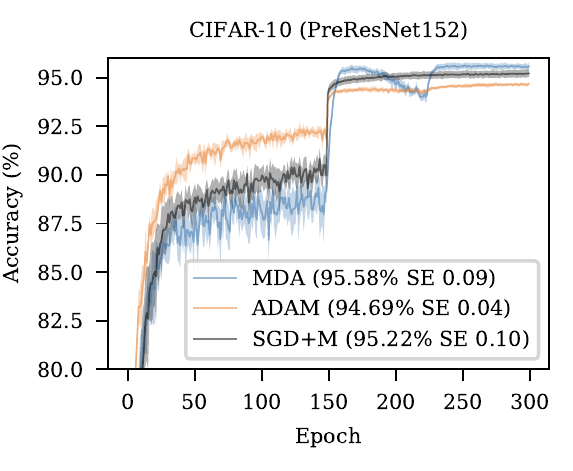}
        \vspace*{-.6cm}
        \caption*{(a)}
     \end{minipage}%
     \begin{minipage}{.5\textwidth}
        \includegraphics[width=.95\linewidth]{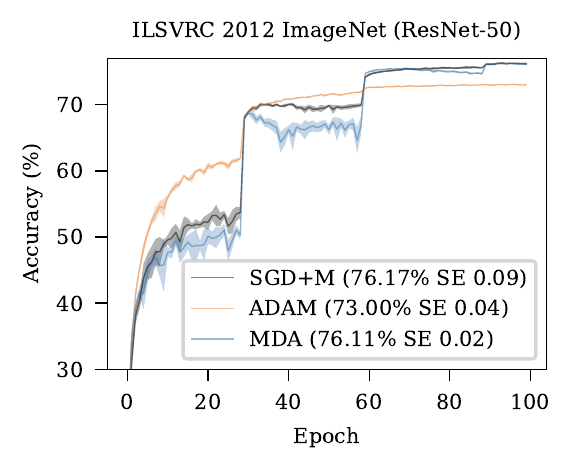}
        \vspace*{-.6cm}
        \caption*{(b)}
     \end{minipage}%

    \begin{minipage}{.5\textwidth}
        \includegraphics[width=.95\linewidth]{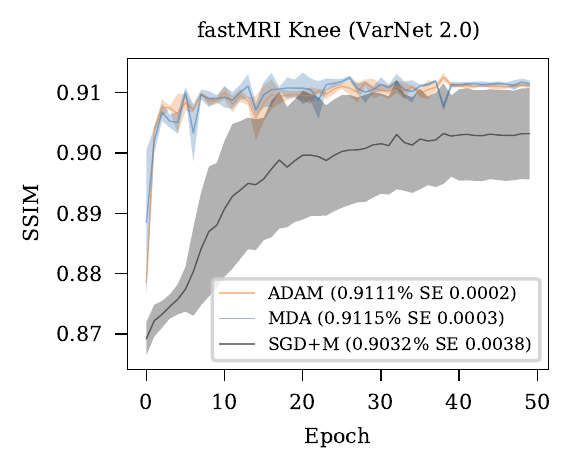}
        \vspace*{-.6cm}
        \caption*{(c)}
     \end{minipage}%
      \begin{minipage}{.5\textwidth}
        \center
        CIFAR-10 Ablations \\
        \vspace*{0.3em}
        
        \begin{tabular}{|c|c|}
        \hline 
        Variants & Test Accuracy\tabularnewline
        \hline 
        \hline 
        Dual averaging & 87.09\%\tabularnewline
        \hline 
        + Momentum & 90.80\%\tabularnewline
        \hline 
        + $\lambda\propto\sqrt{k+1}$ & \textbf{95.58}\%\tabularnewline
        \hline 
        \end{tabular}
        \caption*{(d)}
     \end{minipage}
     
\vspace*{-.3cm}     
     
 \caption{\label{fig:cv}(a): Test accuracy of SGD with momentum and MDA on CIFAR-10 and (b) on ImageNet. (c): Test SSIM of Adam and MDA on fastMRI Knee dataset. MDA matches the performance of the best first-order methods in these computer vision tasks. (d): Ablation study for MDA on CIFAR-10. }

     \vspace*{-.5cm}

\end{figure}

\subsection{MRI reconstruction}
For our MRI reconstruction task we used the fastMRI knee dataset \citep{zbontar2018fastmri}. It consists of more than 10k training examples from approximately 1.5k fully sampled knee MRIs.  The fastMRI challenge is a kind of image-to-image prediction task where the model must predict a MRI image from raw data represented as ``k-space'' image. We trained the VarNet 2.0 model introduced by \cite{sriram2020end}, which is currently the state-of-the-art for this task. We used 12 cascades, batch-size 8, a 4x acceleration factor, 16 center lines and the masking procedure described in \cite{defazio2019offset}. \autoref{fig:cv} (c) shows the SSIM scores obtained for each method. We observe that MDA performs slightly better than Adam, although the difference is within the standard error. SGD performs particularly badly for this task, no mater what tuning of learning rate schedule is tried. The visual difference between reconstructions given by the best SGD trained model, versus the best model from the other two methods is readily apparent when compared at the pixel level (\autoref{fig:reconstructions}).

\begin{figure}[t]
    \begin{minipage}{.5\textwidth}
        \includegraphics[width=.475\linewidth]{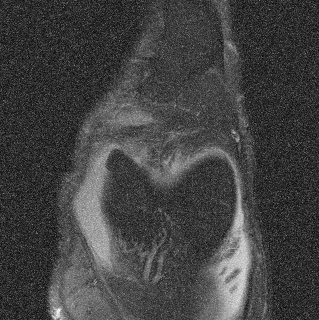}
        \caption*{Ground truth}
     \end{minipage}%
     \begin{minipage}{.5\textwidth}
        \includegraphics[width=.95\linewidth]{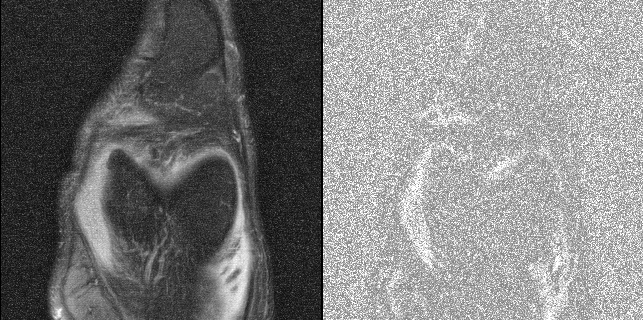}
        \caption*{MDA (SSIM 0.48844}
     \end{minipage}%

    \begin{minipage}{.5\textwidth}
        \includegraphics[width=.95\linewidth]{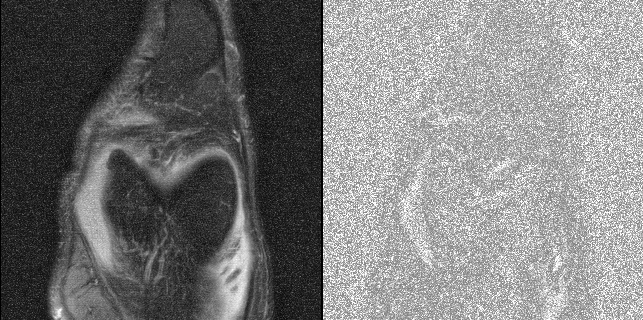}
        \caption*{ADAM (SSIM 0.48744)}
     \end{minipage}%
     \begin{minipage}{.5\textwidth}
        \includegraphics[width=.95\linewidth]{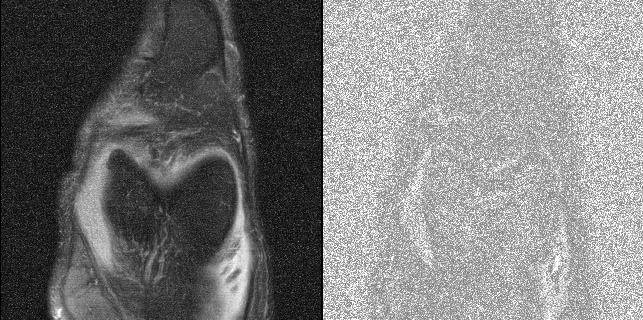}
        \caption*{SGD (SSIM 0.48361)}
     \end{minipage}%

    \captionof{figure}{\label{fig:reconstructions} Reconstruction images for an illustrative knee slice for the same model trained with each of the 3 methods, using the best model for the seeds for each. The difference image between the ground-truth and the noise is shown on the right.}
     \vspace*{-.3cm}

\end{figure}

\subsection{Neural Machine Translation (NMT)}
We run our machine translation task on the IWSLT'14 German-to-English (De-En) dataset (approximately 160k sentence pairs) \citep{cettolo2014report}. We use a Transformer architecture and follow the settings reported in \cite{ott2019fairseq}, using the pre-normalization described in \cite{wang2019learning}. 
The length penalty is set to 0.6 and the beam size is set to 5. For NMT, BLEU score is used \citep{papineni2002bleu}. We report the results of the best checkpoints with respect to the BLEU score averaged over 20 seeds. We report tokenized case-insensitive BLEU. \autoref{fig:nlp} reports the training loss and the BLEU score on the test set of SGD, MDA and Adam on IWSLT'14. SGD as reported in \citep{yao2020adahessian} performs worse than the other methods. While Adam and MDA match in terms of training loss,  MDA outperforms Adam (by 0.20) on the test set. Despite containing no adaptivity, the MDA method is as capable as Adam for this task.
 
\begin{figure}[t]
    \begin{minipage}{.5\textwidth}
    	\centering
        \includegraphics[width=.85\linewidth]{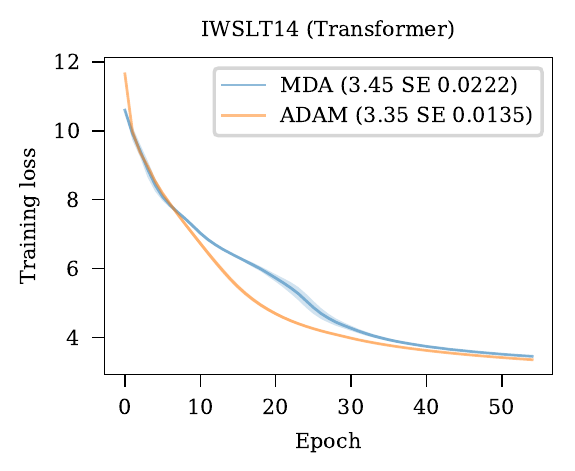}
     \end{minipage}%
       \begin{minipage}{.5\textwidth}
	\centering
        \includegraphics[width=.85\linewidth]{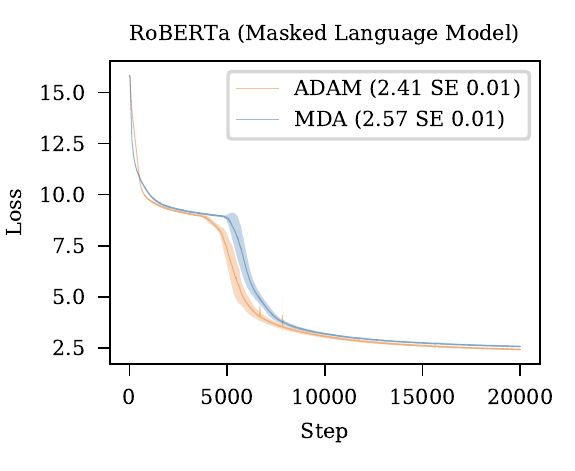}
     \end{minipage}%
   
  \begin{minipage}{0.5\textwidth}\hspace{2em}
        \resizebox{.9\textwidth}{!}{
 \begin{tabular}{|l|c|c|c|}
  \hline
   & SGD & Adam & MDA \\
  \hline
  BLEU & $28.45$ & $34.10\pm0.13$ & $34.18\pm0.11$\\
  \hline
\end{tabular}} 
\end{minipage}
\vspace*{-.2cm}
\caption{Left: performance of Adam and MDA in NMT on IWSLT'14. The plot shows the training loss convergence, while the table provides the BLEU score on the test set. Right: performance of Adam and MDA on RoBERTa training.}\label{fig:nlp}

\end{figure}

\subsection{Masked Language Modeling}
Our largest comparison was on the task of masked language modeling. Pretraining using masked language models has quickly become a standard approach within the natural language processing community \citep{bert}, so it serves as a large-scale, realistic task that is representative of optimization in modern NLP. We used the RoBERTa variant of the BERT model \citep{liu2019roberta}, as implemented in fairseq, training on the concatenation of BookCorpus \citep{zhu2015aligning} and English Wikipedia. \autoref{fig:nlp} shows the training loss for Adam and MDA; SGD fails on this task. MDA's learning curve is virtually identical to Adam. The ``elbow'' shape of the graph is due to the training reaching the end of the first epoch around step 4000. On validation data, MDA achieves a perplexity of 5.3 broadly comparable to the 4.95 value of Adam. As we are using hyper-parameters tuned for Adam, we believe this small gap can be further closed with additional tuning.

\subsection{Ablation study}
As our approach builds upon regular dual averaging, we performed an ablation study on CIFAR-10 to assess the improvement from our changes. We ran each method with a sweep of learning rates, both with flat learning rate schedules and the standard stage-wise scheme. The results are shown in \autoref{fig:cv} (d). Regular dual averaging performs extremely poorly on this task, which may explain why dual averaging variants have seen no use that we are aware of for deep neural network optimization.  The best hyper-parameter combination was LR 1 with the flat LR scheme. We report the results based on the last-iterate, rather than a random iterate (required by the theory), since such post-hoc sampling performs poorly for non-convex tasks. The addition of momentum in the form of iterate averaging within the method brings the test accuracy up by  3.7\%, and allows for the use of a larger learning rate of 2.5. The largest increase is from the use of an increasing lambda sequence, which improves performance by a further 4.78\%.

\section{Tips for usage}
When applying the MDA algorithm to a new problem, we found the following guidelines to be useful:
\begin{itemize}
\item MDA may be slower than other methods at the early iterations. It is important to run the method to convergence when comparing against other methods.
\item The amount of weight decay required when using MDA is often significantly lower than for SGD+M or Adam. We recommend trying a sweep with a maximum of the default for SGD+M or Adam, and a minimum of zero.
\item Learning rates for MDA are much larger than SGD+M or Adam, due to their different parameterizations. When comparing to SGD+M with learning rate $\eta$ and momentum $\beta$, a value comparable to $\eta/(1-\beta)$ is a good starting point. On NLP problems, learning rates as large at 15.0 are sometimes effective.
\end{itemize}

\section{Conclusion}
Based on our experiments, the MDA algorithm appears to be a good choice for a general purpose optimization algorithm for the non-convex problems encountered in deep learning. It avoids the sometimes suboptimal test performance of Adam, while converging on problems where SGD+M fails to work well. Unlike Adam which has no general convergence theory for non-convex problems under the standard hyper-parameter settings, we have proven convergence of MDA under realistic hyper-parameter settings for non-convex problems. It remains an open question why MDA is able to provide the best result in SGD and Adam worlds.
%

\paragraph{Acknowledgements.}
Samy Jelassi would like to thank Walid Krichene, Damien Scieur, Yuanzhi Li,  Robert M. Gower and Michael Rabbat for helpful discussions and Anne Wu for her precious help in setting some numerical experiments. 

\bibliography{main}

\begin{thebibliography}{56}
\providecommand{\natexlab}[1]{#1}
\providecommand{\url}[1]{\texttt{#1}}
\expandafter\ifx\csname urlstyle\endcsname\relax
  \providecommand{\doi}[1]{doi: #1}\else
  \providecommand{\doi}{doi: \begingroup \urlstyle{rm}\Url}\fi

\bibitem[Allen-Zhu(2018)]{smallgradients2018}
Zeyuan Allen-Zhu.
\newblock How to make the gradients small stochastically: Even faster convex
  and nonconvex sgd.
\newblock In S.~Bengio, H.~Wallach, H.~Larochelle, K.~Grauman, N.~Cesa-Bianchi,
  and R.~Garnett (eds.), \emph{Advances in Neural Information Processing
  Systems 31}, pp.\  1157--1167. Curran Associates, Inc., 2018.

\bibitem[Allen-Zhu \& Hazan(2016)Allen-Zhu and Hazan]{blackboxoptimal2016}
Zeyuan Allen-Zhu and Elad Hazan.
\newblock Optimal black-box reductions between optimization objectives.
\newblock In D.~D. Lee, M.~Sugiyama, U.~V. Luxburg, I.~Guyon, and R.~Garnett
  (eds.), \emph{Advances in Neural Information Processing Systems 29}, pp.\
  1614--1622. Curran Associates, Inc., 2016.

\bibitem[Bauschke et~al.(1997)Bauschke, Borwein, et~al.]{bauschke1997legendre}
Heinz~H Bauschke, Jonathan~M Borwein, et~al.
\newblock Legendre functions and the method of random bregman projections.
\newblock \emph{Journal of convex analysis}, 4\penalty0 (1):\penalty0 27--67,
  1997.

\bibitem[Berrada et~al.(2018)Berrada, Zisserman, and Kumar]{berrada2018deep}
Leonard Berrada, Andrew Zisserman, and M~Pawan Kumar.
\newblock Deep frank-wolfe for neural network optimization.
\newblock \emph{arXiv preprint arXiv:1811.07591}, 2018.

\bibitem[Bos \& Chug(1996)Bos and Chug]{bos1996using}
Siegfried Bos and E~Chug.
\newblock Using weight decay to optimize the generalization ability of a
  perceptron.
\newblock In \emph{Proceedings of International Conference on Neural Networks
  (ICNN'96)}, volume~1, pp.\  241--246. IEEE, 1996.

\bibitem[Bottou(1991)]{bottou1991stochastic}
L~Bottou.
\newblock Stochastic gradient learning in neural networks.
\newblock \emph{Proceedings of Neuro-Nimes 91, Nimes, France}, 1991.

\bibitem[Bottou \& Bousquet(2008)Bottou and Bousquet]{bottou2008tradeoffs}
L{\'e}on Bottou and Olivier Bousquet.
\newblock The tradeoffs of large scale learning.
\newblock In \emph{Advances in neural information processing systems}, pp.\
  161--168, 2008.

\bibitem[Bottou et~al.(2016)Bottou, Curtis, and
  Nocedal]{bottou2016optimization}
L{\'e}on Bottou, Frank~E. Curtis, and Jorge Nocedal.
\newblock Optimization methods for large-scale machine learning, 2016.

\bibitem[Bregman(1967)]{bregman1967relaxation}
Lev~M Bregman.
\newblock The relaxation method of finding the common point of convex sets and
  its application to the solution of problems in convex programming.
\newblock \emph{USSR computational mathematics and mathematical physics},
  7\penalty0 (3):\penalty0 200--217, 1967.

\bibitem[Cettolo et~al.(2014)Cettolo, Niehues, St{\"u}ker, Bentivogli, and
  Federico]{cettolo2014report}
Mauro Cettolo, Jan Niehues, Sebastian St{\"u}ker, Luisa Bentivogli, and
  Marcello Federico.
\newblock Report on the 11th iwslt evaluation campaign, iwslt 2014.
\newblock In \emph{Proceedings of the Eleventh International Workshop on Spoken
  Language Translation (IWSLT 2014)}, 2014.

\bibitem[Choi et~al.(2019)Choi, Shallue, Nado, Lee, Maddison, and
  Dahl]{choi2019empirical}
Dami Choi, Christopher~J Shallue, Zachary Nado, Jaehoon Lee, Chris~J Maddison,
  and George~E Dahl.
\newblock On empirical comparisons of optimizers for deep learning.
\newblock \emph{arXiv preprint arXiv:1910.05446}, 2019.

\bibitem[Colin et~al.(2016)Colin, Bellet, Salmon, and
  Cl{\'e}men{\c{c}}on]{colin2016gossip}
Igor Colin, Aur{\'e}lien Bellet, Joseph Salmon, and St{\'e}phan
  Cl{\'e}men{\c{c}}on.
\newblock Gossip dual averaging for decentralized optimization of pairwise
  functions.
\newblock \emph{arXiv preprint arXiv:1606.02421}, 2016.

\bibitem[Davis \& Drusvyatskiy(2019)Davis and
  Drusvyatskiy]{davis2019stochastic}
Damek Davis and Dmitriy Drusvyatskiy.
\newblock Stochastic model-based minimization of weakly convex functions.
\newblock \emph{SIAM Journal on Optimization}, 29\penalty0 (1):\penalty0
  207--239, 2019.

\bibitem[Defazio(2019)]{defazio2019offset}
Aaron Defazio.
\newblock Offset sampling improves deep learning based accelerated mri
  reconstructions by exploiting symmetry, 2019.

\bibitem[Defazio(2020)]{defazio2020mom}
Aaron Defazio.
\newblock Understanding the role of momentum in non-convex optimization:
  Practical insights from a lyapunov analysis.
\newblock \emph{arXiv preprint}, 2020.

\bibitem[Defazio \& Bottou(2019)Defazio and Bottou]{defazio2019ineffectiveness}
Aaron Defazio and L{\'e}on Bottou.
\newblock On the ineffectiveness of variance reduced optimization for deep
  learning.
\newblock In \emph{Advances in Neural Information Processing Systems}, pp.\
  1755--1765, 2019.

\bibitem[D{\'e}fossez et~al.(2020)D{\'e}fossez, Bottou, Bach, and
  Usunier]{defossez2020convergence}
Alexandre D{\'e}fossez, L{\'e}on Bottou, Francis Bach, and Nicolas Usunier.
\newblock On the convergence of adam and adagrad.
\newblock \emph{arXiv preprint arXiv:2003.02395}, 2020.

\bibitem[Devlin et~al.(2019)Devlin, Chang, Lee, and Toutanova]{bert}
Jacob Devlin, Ming-Wei Chang, Kenton Lee, and Kristina Toutanova.
\newblock {BERT}: Pre-training of deep bidirectional transformers for language
  understanding.
\newblock In \emph{Proceedings of the 2019 Conference of the North {A}merican
  Chapter of the Association for Computational Linguistics: Human Language
  Technologies, Volume 1 (Long and Short Papers)}, Minneapolis, Minnesota, June
  2019. Association for Computational Linguistics.

\bibitem[Duchi \& Ruan(2016)Duchi and Ruan]{duchi2016asymptotic}
John Duchi and Feng Ruan.
\newblock Asymptotic optimality in stochastic optimization.
\newblock \emph{arXiv preprint arXiv:1612.05612}, 2016.

\bibitem[Duchi et~al.(2011{\natexlab{a}})Duchi, Hazan, and
  Singer]{duchi2011adaptive}
John Duchi, Elad Hazan, and Yoram Singer.
\newblock Adaptive subgradient methods for online learning and stochastic
  optimization.
\newblock \emph{Journal of machine learning research}, 12\penalty0 (7),
  2011{\natexlab{a}}.

\bibitem[Duchi et~al.(2011{\natexlab{b}})Duchi, Agarwal, and
  Wainwright]{duchi2011dual}
John~C Duchi, Alekh Agarwal, and Martin~J Wainwright.
\newblock Dual averaging for distributed optimization: Convergence analysis and
  network scaling.
\newblock \emph{IEEE Transactions on Automatic control}, 57\penalty0
  (3):\penalty0 592--606, 2011{\natexlab{b}}.

\bibitem[Eckstein(1993)]{eckstein1993nonlinear}
Jonathan Eckstein.
\newblock Nonlinear proximal point algorithms using bregman functions, with
  applications to convex programming.
\newblock \emph{Mathematics of Operations Research}, 18\penalty0 (1):\penalty0
  202--226, 1993.

\bibitem[He et~al.(2016)He, Zhang, Ren, and Sun]{he2016deep}
Kaiming He, Xiangyu Zhang, Shaoqing Ren, and Jian Sun.
\newblock Deep residual learning for image recognition.
\newblock In \emph{Proceedings of the IEEE conference on computer vision and
  pattern recognition}, pp.\  770--778, 2016.

\bibitem[Hinton et~al.(2012)Hinton, Srivastava, and Swersky]{hinton2012neural}
Geoffrey Hinton, Nitish Srivastava, and Kevin Swersky.
\newblock Neural networks for machine learning lecture 6a overview of
  mini-batch gradient descent, 2012.

\bibitem[Hosseini et~al.(2013)Hosseini, Chapman, and
  Mesbahi]{hosseini2013online}
Saghar Hosseini, Airlie Chapman, and Mehran Mesbahi.
\newblock Online distributed optimization via dual averaging.
\newblock In \emph{52nd IEEE Conference on Decision and Control}, pp.\
  1484--1489. IEEE, 2013.

\bibitem[Juditsky et~al.(2019)Juditsky, Kwon, and
  Moulines]{juditsky2019unifying}
Anatoli Juditsky, Joon Kwon, and {\'E}ric Moulines.
\newblock Unifying mirror descent and dual averaging.
\newblock \emph{arXiv preprint arXiv:1910.13742}, 2019.

\bibitem[Kingma \& Ba(2014)Kingma and Ba]{kingma2014adam}
Diederik~P Kingma and Jimmy Ba.
\newblock Adam: A method for stochastic optimization.
\newblock \emph{arXiv preprint arXiv:1412.6980}, 2014.

\bibitem[Krogh \& Hertz(1992)Krogh and Hertz]{krogh1992simple}
Anders Krogh and John~A Hertz.
\newblock A simple weight decay can improve generalization.
\newblock In \emph{Advances in neural information processing systems}, pp.\
  950--957, 1992.

\bibitem[LeCun et~al.(1998)LeCun, Bottou, Bengio, and
  Haffner]{lecun1998gradient}
Yann LeCun, L{\'e}on Bottou, Yoshua Bengio, and Patrick Haffner.
\newblock Gradient-based learning applied to document recognition.
\newblock \emph{Proceedings of the IEEE}, 86\penalty0 (11):\penalty0
  2278--2324, 1998.

\bibitem[Lee \& Wright(2012)Lee and Wright]{lee2012manifold}
Sangkyun Lee and Stephen~J Wright.
\newblock Manifold identification in dual averaging for regularized stochastic
  online learning.
\newblock \emph{The Journal of Machine Learning Research}, 13\penalty0
  (1):\penalty0 1705--1744, 2012.

\bibitem[Li \& Orabona(2019)Li and Orabona]{li2019convergence}
Xiaoyu Li and Francesco Orabona.
\newblock On the convergence of stochastic gradient descent with adaptive
  stepsizes.
\newblock In \emph{The 22nd International Conference on Artificial Intelligence
  and Statistics}, pp.\  983--992, 2019.

\bibitem[Liu et~al.(2019)Liu, Ott, Goyal, Du, Joshi, Chen, Levy, Lewis,
  Zettlemoyer, and Stoyanov]{liu2019roberta}
Yinhan Liu, Myle Ott, Naman Goyal, Jingfei Du, Mandar Joshi, Danqi Chen, Omer
  Levy, Mike Lewis, Luke Zettlemoyer, and Veselin Stoyanov.
\newblock Roberta: A robustly optimized bert pretraining approach.
\newblock \emph{arXiv preprint arXiv:1907.11692}, 2019.

\bibitem[Merity et~al.(2016)Merity, Xiong, Bradbury, and
  Socher]{merity2016pointer}
Stephen Merity, Caiming Xiong, James Bradbury, and Richard Socher.
\newblock Pointer sentinel mixture models.
\newblock \emph{arXiv preprint arXiv:1609.07843}, 2016.

\bibitem[Nesterov \& Shikhman(2015)Nesterov and Shikhman]{nesterov2015quasi}
Yu~Nesterov and Vladimir Shikhman.
\newblock Quasi-monotone subgradient methods for nonsmooth convex minimization.
\newblock \emph{Journal of Optimization Theory and Applications}, 165\penalty0
  (3):\penalty0 917--940, 2015.

\bibitem[Nesterov(2009)]{nesterov2009primal}
Yurii Nesterov.
\newblock Primal-dual subgradient methods for convex problems.
\newblock \emph{Mathematical programming}, 120\penalty0 (1):\penalty0 221--259,
  2009.

\bibitem[Nesterov(2013)]{Nesterov-convex}
Yurii Nesterov.
\newblock \emph{Introductory lectures on convex optimization: A basic course}.
\newblock Springer, 2013.

\bibitem[Ott et~al.(2019)Ott, Edunov, Baevski, Fan, Gross, Ng, Grangier, and
  Auli]{ott2019fairseq}
Myle Ott, Sergey Edunov, Alexei Baevski, Angela Fan, Sam Gross, Nathan Ng,
  David Grangier, and Michael Auli.
\newblock fairseq: A fast, extensible toolkit for sequence modeling.
\newblock In \emph{Proceedings of NAACL-HLT 2019: Demonstrations}, 2019.

\bibitem[Papineni et~al.(2002)Papineni, Roukos, Ward, and
  Zhu]{papineni2002bleu}
Kishore Papineni, Salim Roukos, Todd Ward, and Wei-Jing Zhu.
\newblock Bleu: a method for automatic evaluation of machine translation.
\newblock In \emph{Proceedings of the 40th annual meeting of the Association
  for Computational Linguistics}, pp.\  311--318, 2002.

\bibitem[Robbins \& Monro(1951)Robbins and Monro]{robbins1951stochastic}
Herbert Robbins and Sutton Monro.
\newblock A stochastic approximation method.
\newblock \emph{The annals of mathematical statistics}, pp.\  400--407, 1951.

\bibitem[Sebbouh et~al.(2020)Sebbouh, Gower, and
  Defazio]{sebbouh2020convergence}
Othmane Sebbouh, Robert~M Gower, and Aaron Defazio.
\newblock On the convergence of the stochastic heavy ball method.
\newblock \emph{arXiv preprint arXiv:2006.07867}, 2020.

\bibitem[Shahrampour \& Jadbabaie(2013)Shahrampour and
  Jadbabaie]{shahrampour2013exponentially}
Shahin Shahrampour and Ali Jadbabaie.
\newblock Exponentially fast parameter estimation in networks using distributed
  dual averaging.
\newblock In \emph{52nd IEEE Conference on Decision and Control}, pp.\
  6196--6201. IEEE, 2013.

\bibitem[Sriram et~al.(2020)Sriram, Zbontar, Murrell, Defazio, Zitnick,
  Yakubova, Knoll, and Johnson]{sriram2020end}
Anuroop Sriram, Jure Zbontar, Tullie Murrell, Aaron Defazio, C~Lawrence
  Zitnick, Nafissa Yakubova, Florian Knoll, and Patricia Johnson.
\newblock End-to-end variational networks for accelerated mri reconstruction.
\newblock \emph{arXiv preprint arXiv:2004.06688}, 2020.

\bibitem[Sutskever et~al.(2013)Sutskever, Martens, Dahl, and
  Hinton]{sutskever2013importance}
Ilya Sutskever, James Martens, George Dahl, and Geoffrey Hinton.
\newblock On the importance of initialization and momentum in deep learning.
\newblock In \emph{International conference on machine learning}, pp.\
  1139--1147, 2013.

\bibitem[Tao et~al.(2018)Tao, Pan, Wu, and Tao]{tao2018primal}
Wei Tao, Zhisong Pan, Gaowei Wu, and Qing Tao.
\newblock Primal averaging: A new gradient evaluation step to attain the
  optimal individual convergence.
\newblock \emph{IEEE transactions on cybernetics}, 50\penalty0 (2):\penalty0
  835--845, 2018.

\bibitem[Teboulle(1992)]{teboulle1992entropic}
Marc Teboulle.
\newblock Entropic proximal mappings with applications to nonlinear
  programming.
\newblock \emph{Mathematics of Operations Research}, 17\penalty0 (3):\penalty0
  670--690, 1992.

\bibitem[Tsianos et~al.(2012)Tsianos, Lawlor, and Rabbat]{tsianos2012push}
Konstantinos~I Tsianos, Sean Lawlor, and Michael~G Rabbat.
\newblock Push-sum distributed dual averaging for convex optimization.
\newblock In \emph{2012 ieee 51st ieee conference on decision and control
  (cdc)}, pp.\  5453--5458. IEEE, 2012.

\bibitem[Wang et~al.(2019)Wang, Li, Xiao, Zhu, Li, Wong, and
  Chao]{wang2019learning}
Qiang Wang, Bei Li, Tong Xiao, Jingbo Zhu, Changliang Li, Derek~F Wong, and
  Lidia~S Chao.
\newblock Learning deep transformer models for machine translation.
\newblock \emph{arXiv preprint arXiv:1906.01787}, 2019.

\bibitem[Ward et~al.(2019)Ward, Wu, and Bottou]{ward2019adagrad}
Rachel Ward, Xiaoxia Wu, and Leon Bottou.
\newblock Adagrad stepsizes: sharp convergence over nonconvex landscapes.
\newblock In \emph{International Conference on Machine Learning}, pp.\
  6677--6686, 2019.

\bibitem[Wei et~al.(2019)Wei, Lee, Liu, and Ma]{wei2019regularization}
Colin Wei, Jason~D Lee, Qiang Liu, and Tengyu Ma.
\newblock Regularization matters: Generalization and optimization of neural
  nets vs their induced kernel.
\newblock In \emph{Advances in Neural Information Processing Systems}, pp.\
  9712--9724, 2019.

\bibitem[Wilson et~al.(2017)Wilson, Roelofs, Stern, Srebro, and
  Recht]{wilson2017marginal}
Ashia~C Wilson, Rebecca Roelofs, Mitchell Stern, Nati Srebro, and Benjamin
  Recht.
\newblock The marginal value of adaptive gradient methods in machine learning.
\newblock In \emph{Advances in neural information processing systems}, pp.\
  4148--4158, 2017.

\bibitem[Xiao(2010)]{xiao2010dual}
Lin Xiao.
\newblock Dual averaging methods for regularized stochastic learning and online
  optimization.
\newblock \emph{Journal of Machine Learning Research}, 11\penalty0
  (Oct):\penalty0 2543--2596, 2010.

\bibitem[Yao et~al.(2020)Yao, Gholami, Shen, Keutzer, and
  Mahoney]{yao2020adahessian}
Zhewei Yao, Amir Gholami, Sheng Shen, Kurt Keutzer, and Michael~W Mahoney.
\newblock Adahessian: An adaptive second order optimizer for machine learning.
\newblock \emph{arXiv preprint arXiv:2006.00719}, 2020.

\bibitem[Zbontar et~al.(2018)Zbontar, Knoll, Sriram, Muckley, Bruno, Defazio,
  Parente, Geras, Katsnelson, Chandarana, et~al.]{zbontar2018fastmri}
Jure Zbontar, Florian Knoll, Anuroop Sriram, Matthew~J Muckley, Mary Bruno,
  Aaron Defazio, Marc Parente, Krzysztof~J Geras, Joe Katsnelson, Hersh
  Chandarana, et~al.
\newblock fastmri: An open dataset and benchmarks for accelerated mri.
\newblock \emph{arXiv preprint arXiv:1811.08839}, 2018.

\bibitem[Zhang \& He(2018)Zhang and He]{zhang2018convergence}
Siqi Zhang and Niao He.
\newblock On the convergence rate of stochastic mirror descent for nonsmooth
  nonconvex optimization.
\newblock \emph{arXiv preprint arXiv:1806.04781}, 2018.

\bibitem[Zhu et~al.(2015)Zhu, Kiros, Zemel, Salakhutdinov, Urtasun, Torralba,
  and Fidler]{zhu2015aligning}
Yukun Zhu, Ryan Kiros, Rich Zemel, Ruslan Salakhutdinov, Raquel Urtasun,
  Antonio Torralba, and Sanja Fidler.
\newblock Aligning books and movies: Towards story-like visual explanations by
  watching movies and reading books.
\newblock In \emph{Proceedings of the IEEE international conference on computer
  vision}, pp.\  19--27, 2015.

\bibitem[Zou et~al.(2018)Zou, Shen, Jie, Sun, and Liu]{zou2018weighted}
Fangyu Zou, Li~Shen, Zequn Jie, Ju~Sun, and Wei Liu.
\newblock Weighted adagrad with unified momentum.
\newblock \emph{arXiv preprint arXiv:1808.03408}, 2018.

\end{thebibliography}
\bibliographystyle{iclr2021_conference}
\pagebreak
\appendix

\paragraph{Notation.} Let $x_0,...,x_k$ be the iterates returned by a stochastic algorithm. We use $\mathcal{F}_k$ to refer to the filtration with respect to $x_0,...,x_k$. For a random variable $Y,$ $\mathbb{E}_k[Y]$ denote the expectation of a random variable $Y$ conditioned on $\mathcal{F}_k$ i.e.\ $\mathbb{E}_k[Y] = \mathbb{E}[Y|\mathcal{F}_k].$

\begin{lemma}[LEMMA 1.2.3, \cite{Nesterov-convex}]\label{lem:smooth}
Suppose that $f$ is differentiable and has $L$-Lipschitz gradients, then: 
   \begin{align}\label{eq:smoothf}
   \left| f(x)-f(y)-\langle \nabla f(y),x-y\rangle \right| \leq \frac{L}{2}\|x-y\|_2^2, \qquad \forall x,y\in \mathbb{R}^n.
\end{align}
\end{lemma}

\section{Convergence analysis of non-convex SGD+M}\label{sec:sgdmom_proof}

We remind that the SGD+M algorithm is commonly written in the following form
\begin{equation}
    \begin{split}
        m_{k+1}&=\beta_km_k+\nabla f(x_k,\xi_k),\\
        x_{k+1}&=x_k-\alpha_k m_{k+1},
    \end{split}
\end{equation}
where $x_k$ is the iterate sequence, and $m_k$ is the momentum buffer. Instead, we will make use of the averaging form of the momentum method \cite{sebbouh2020convergence} also known as the stochastic primal averaging (SPA) form \cite{tao2018primal}:
\begin{equation}
    \begin{split}
    z_{k+1}&=z_k-\eta_k\nabla f(x_k,\xi_k),\\
    x_{k+1}&=(1-c_{k+1})x_k+c_{k+1}z_{k+1}.
    \end{split}
\end{equation}

For specific choices of values for the hyper-parameters,the $x_k$ sequence generated by this method will be identical to that of SGD+M. We make use of the convergence analysis of non-convex SPA by \citet{defazio2020mom}.
\begin{theorem}\label{thm:sgd_rate}
Let $f$ be a $L$-smooth function. For a fixed step $k$, let $\eta_k>0$ be the stepsize and $c_k$ the averaging parameter in SPA. Let $x_k$ and $z_k$ be the iterates in SPA. Then, we have:
  \begin{equation}\label{eq:sgd_rate}
 \begin{split}
\frac{1}{2\eta_k}\|\nabla f(x_k)\|_2^2+\frac{1}{2\eta_k}\|\nabla f(z_k)\|_2^2&\leq \Gamma_k-\mathbb{E}_{k}[\Gamma_{k+1}] +L\mathbb{E}_{k}[\|\nabla f(x_k,\xi_k)\|_2^2]\\
&+\frac{1}{2}\left[\frac{1}{\eta_k^2}\left(\frac{1}{c_k}-1+\eta_kL\right)\left(\frac{1}{c_k} -1\right)+\frac{L}{\eta_k}\left(\frac{1}{c_k} -1\right)^2\right.\\
&\left. -\frac{1}{\eta_{k-1}^2c_k^2}\right] L \|x_{k}-x_{k-1}\|_2^2,
\end{split}
 \end{equation}
 where $\Gamma_k$ is the Lyapunov function defined as 
 \begin{align}
     \Gamma_k&=\frac{1}{\eta_k^2}f(z_{k+1}) + \frac{L}{\eta_k}\left(\frac{1}{c_k} -1\right)f(x_k)+\frac{L}{2\eta_k c_{k+1}^2}\|x_{k+1}-x_k\|_2^2.
 \end{align}
\end{theorem}
 

\section{Convergence analysis of MDA}\label{sec:mda_proof}

This section is dedicated to the convergence proof of MDA. To obtain the rate in \autoref{thm:mda}, we use \autoref{prop:mdda} and \autoref{thm:sgd_rate}. We start by introducing some notations. 

\subsection{Notations and useful inequalities}\label{sec:notations}
At a fixed step step $k$, we remind that the function $h^{(k)}\colon \mathbb{R}^n\rightarrow \mathbb{R}$ is defined as
\begin{align}\label{eq:hk}
      h^{(k)}(x)&=f(x)+\frac{\alpha_k}{2}\|x-x_0\|^2.
 \end{align}
 We now introduce the following notions induced by $ h^{(k)}.$ $\Gamma_{k+1}$ is the Lyapunov function with respect to $h^{(k)}$ and is defined as 
\begin{align}\label{eq:lyap_def}
    \Gamma_{k+1}:=\frac{1}{\eta_{k}^2}h^{(k)}(z_{k+1}) +\frac{L_{h^{(k)}}}{\eta_k}\left(\frac{1}{c_k}-1\right)h^{(k)}(x_k)+\frac{1}{2}\frac{L_{h^{(k)}}}{\eta_k^2 c_{k+1}^2}\|x_{k+1}-x_k\|_2^2,
\end{align}
 $\nabla h^{(k)}(x_k,\xi_k)$ is the stochastic gradient of $h^{(k)}$ and is equal to
 \begin{align}\label{eq:nablah}
     \nabla h^{(k)}(x_k,\xi_k):=\nabla f(x_k,\xi_k)+\alpha_k(x_k-x_0),
 \end{align}
 
and $L_{h^{(k)}}$ is the smoothness constant of $h^{(k)}$,
 \begin{align}\label{eq:Lhk}
     L_{h^{(k)}}:=L+\alpha_k.
 \end{align}
 As our proof assumes that $\eta_k$ and $c_k$ are constant, we remind that our parameters choices in \autoref{alg:DA_momentum} are 
 \begin{align}\label{eq:remind_params_1}
     \eta_k = \eta \sqrt{k+1} \quad \text{ and } \quad \beta_k =\sqrt{k+1},
 \end{align}
 where $\eta>0.$
As a consequence, we have:
\begin{align}\label{eq:remind_params}
    \alpha_k = \frac{\sqrt{k+2}-\sqrt{k+1}}{\eta\sqrt{k+2}}.
\end{align}
Therefore, $ \alpha_k $ is a non-increasing sequence and as a consequence, \begin{align}\label{eq:decrease_alpha}
\alpha_{k+1}-\alpha_k\leq 0.
\end{align}
It can further be shown that 
\begin{align}\label{eq:alphakk1diff}
    \alpha_{k}(\alpha_{k+1}-\alpha_k)\leq 0,
\end{align}
and
\begin{align}\label{eq:alphakk1diff1}
    \alpha_{k+1}(\alpha_{k+1}-\alpha_k)\leq 0.
\end{align}
Moreover, by using the inequality $\sqrt{k+2}-\sqrt{k+1} \leq (2\sqrt{k+1})^{-1}$, we have:
\begin{align}
    \alpha_k \leq \frac{1}{2\eta(k+1)}.
\end{align}
We will encounter the quantity $\alpha_{k+1}-\alpha_k$ in the proof and would like to upper bound it. We first start by upper bounding $\alpha_{k+1}-\alpha_k:$
\begin{align}
    \alpha_{k+1}-\alpha_k&=  \frac{\sqrt{k+3}-\sqrt{k+2}}{\eta\sqrt{k+3}}- \frac{\sqrt{k+2}-\sqrt{k+1}}{ \eta\sqrt{k+2}} \nonumber\\
    &\leq \frac{1}{2\eta\sqrt{(k+3)(k+2)}}- \frac{1}{2 \eta(k+2)}\label{eq:23}\\
    &= \frac{1}{2\eta \sqrt{k+2}}\frac{\sqrt{k+2}-\sqrt{k+3}}{\sqrt{(k+2)(k+3)}} \nonumber\\
    &\leq -\frac{1}{4\eta (k+2)^{3/2}\sqrt{k+3}},\label{eq:24}
\end{align}
where we used the inequalities for any $x>1,$ $\sqrt{x-1}-\sqrt{x}\leq -(2\sqrt{x})^{-1}$ in \eqref{eq:23} and \eqref{eq:24} and for any $x>0,$ $\sqrt{x+1}-\sqrt{x}\leq (2\sqrt{x})^{-1}$ in \eqref{eq:23}.

\subsection{Proof of \autoref{thm:mda}}

\paragraph{Adapting the proof of non-convex SPA.} We start off by using the inequality satisfied by non-convex SPA (\autoref{thm:sgd_rate}). By applying it to $ h^{(k)}$, we obtain:
\begin{equation}\label{eq:daexpupdsmooth}
\begin{split}
        \frac{1}{2\eta_k}\|\nabla h^{(k)}(x_k)\|_2^2+\frac{1}{2\eta_k}\|\nabla h^{(k)}(z_k)\|_2^2 &\leq \Gamma_k-\mathbb{E}_{k}[\Gamma_{k+1}]+  L_{h^{(k)}}  \mathbb{E}_{k}[\|\nabla h^{(k)}(x_k,\xi_k)\|_2^2]\\
    &+\frac{1}{2}\left[\frac{1}{\eta_k^2}\left(\frac{1}{c_k}-1+\eta_kL_{h^{(k)}}\right)\left(\frac{1}{c_k}-1\right)\right.\\
    &\left.+\frac{1}{\eta_k}L_{h^{(k)}}\left(\frac{1}{c_k}-1\right)^2-\frac{1}{\eta_{k-1}^2c_k^2}\right]L_{h^{(k)}}\|x_k-x_{k-1}\|_2^2.
\end{split}
\end{equation}

By using \eqref{eq:nablah} and the inequality $\|a+b\|_2^2 \leq 2\|a\|_2^2 +2\|b\|_2^2$ for $a,b\in\mathbb{R}^n$ in \eqref{eq:daexpupdsmooth}, we obtain:
 \begin{equation}\label{eq:daexpupdsmoothineq}
\begin{split}
        \frac{1}{4\eta_k}\|\nabla f(x_k)\|_2^2+\frac{1}{4\eta_k}\|\nabla f(z_k)\|_2^2 &\leq \Gamma_k-\mathbb{E}_{k}[\Gamma_{k+1}]+  L_{h^{(k)}}  \mathbb{E}_{k}[\|\nabla f(x_k,\xi_k)\|_2^2]\\
        &+L_{h^{(k)}} \alpha_k\|x_k-x_0\|_2^2+\frac{\alpha_k}{2\eta_k}\left(\|x_k-x_0\|_2^2+\|z_k-x_0\|_2^2\right)\\
    &+\frac{1}{2}\left[\frac{1}{\eta_k^2}\left(\frac{1}{c_k}-1+\eta_kL_{h^{(k)}}\right)\left(\frac{1}{c_k}-1\right)\right.\\
    &\left.+\frac{1}{\eta_k}L_{h^{(k)}}\left(\frac{1}{c_k}-1\right)^2-\frac{1}{\eta_{k-1}^2c_k^2}\right]L_{h^{(k)}}\|x_k-x_{k-1}\|_2^2,
\end{split}
\end{equation}
Now, by using the definition \eqref{eq:Lhk} of $ L_{h^{(k)}}$, the choice of parameters \eqref{eq:remind_params_1}, the boundedness assumption (\autoref{ass:bounddomain}) and the properties on the stochastic oracle (\autoref{ass:sto_grad}), \eqref{eq:daexpupdsmoothineq} becomes: 
 \begin{equation}\label{eq:finalbeforelyap}
\begin{split}
        &\frac{1}{4\eta}\|\nabla f(x_k)\|_2^2+\frac{1}{4\eta}\|\nabla f(z_k)\|_2^2\\
        \leq& \Gamma_k-\mathbb{E}_{k}[\Gamma_{k+1}]+  \left(L +\frac{1}{2\eta(k+1)}\right)\sigma^2+ \frac{R^2}{2\eta(k+1)}\left(\frac{1}{\eta}+L+\frac{1}{2\eta(k+1)} \right)\\
    +&\left(\frac{L}{2}+ \frac{1}{4\eta(k+1)}\right)\left[\frac{1}{\eta^2}\left(\frac{1}{c}-1+\eta\left(L+\frac{1}{2\eta(k+1)}\right)\right)\left(\frac{1}{c}-1\right)\right.\\
    &\left.+\frac{1}{\eta}\left(L+\frac{1}{2\eta(k+1)}\right)\left(\frac{1}{c}-1\right)^2-\frac{1}{\eta^2c^2}\right]\|x_k-x_{k-1}\|_2^2.
\end{split}
\end{equation}

\paragraph{Bound on the $\|x_k-x_{k-1}\|_2^2$ term.} We now expand the condition on the stepsize parameter $\eta$ that yields a negative factor in the $\|x_k-x_{k-1}\|_2^2$ term in \eqref{eq:finalbeforelyap}. 
\begin{align}
    \frac{1}{\eta^2}\left(\frac{1}{c}-1+\eta\left(L+\frac{1}{2\eta(k+1)}\right)\right)\left(\frac{1}{c}-1\right)+\frac{1}{\eta}\left(L+\frac{1}{2\eta(k+1)}\right)\left(\frac{1}{c}-1\right)^2 &\leq \frac{1}{\eta^2c^2}\nonumber\\
    \left( \frac{1}{c}-1\right)^2 +  \eta \left( \frac{1}{c}-1\right)\left(L+\frac{1}{2\eta(k+1)}\right)+\eta L \left(\frac{1}{c}-1\right)^2+\frac{1}{2 (k+1)}\left(\frac{1}{c}-1\right)^2 &\leq \frac{1}{c^2},
\end{align}
which leads to 
\begin{align}
      \eta \frac{L}{c}\left( \frac{1}{c}-1 \right) &\leq \left(\frac{2}{c}-1\right) -\left( \frac{1}{c}-1\right)\frac{1}{2c(k+1)}\nonumber\\
    \eta &\leq \frac{1}{L}\left( c+ \frac{1}{2} \right).\label{eq:cond_lambda}
\end{align}

In what follows, we set the stepsize parameter $\eta$ such that it satisfies \eqref{eq:cond_lambda}. The rest of the proof is dedicated to bouding the difference of Lyapunov functions in \eqref{eq:finalbeforelyap}. 

 \paragraph{Bound on the difference of Lyapunov functions.}  By using \eqref{eq:lyap_def}, the difference of Lyapunov functions is:
\begin{align}
    \Gamma_k-\mathbb{E}_{k}[\Gamma_{k+1}]&=\frac{1}{\eta_{k-1}^2}h^{(k)}(z_{k})-\frac{1}{\eta_{k}^2}\mathbb{E}_{k}[h^{(k)}(z_{k+1})]\label{eq:28}\\
    &+\frac{L_{h^{(k)}}}{\eta_{k-1}}\left(\frac{1}{c_{k-1}}-1\right)h^{(k)}(x_{k-1})-\frac{L_{h^{(k)}}}{\eta_k}\left(\frac{1}{c_k}-1\right)h^{(k)}(x_k)\label{eq:29}\\
    &+\frac{1}{2}\frac{L_{h^{(k)}}}{\eta_{k-1}^2 c_{k}^2}\|x_{k}-x_{k-1}\|_2^2-\frac{1}{2}\frac{L_{h^{(k)}}}{\eta_k^2 c_{k+1}^2}\|x_{k+1}-x_k\|_2^2.\label{eq:30}
\end{align}
We now expand each term in the equality above. We start off by looking at \eqref{eq:28}. By using the definition of $h^{(k)}$, this latter can be rewritten as: 
\begin{equation}
    \begin{split}\label{eq:new_28}
        &\frac{1}{\eta_{k-1}^2}h^{(k)}(z_{k})-\frac{1}{\eta_{k}^2}\mathbb{E}_{k}[h^{(k)}(z_{k+1})]\\
        =&\frac{1}{\eta^2}(f(z_{k})-f(z_{k+1}))+\frac{\alpha_k\|z_k-x_0\|_2^2-\alpha_{k+1}\|z_{k+1}-x_0\|_2^2}{\eta^2}\\
        &+\frac{\alpha_{k+1}-\alpha_{k}}{\eta^2}\|z_{k+1}-x_0\|_2^2\\
        \leq& \frac{1}{\eta^2}(f(z_{k})-f(z_{k+1}))+\frac{\alpha_k\|z_k-x_0\|_2^2-\alpha_{k+1}\|z_{k+1}-x_0\|_2^2}{\eta^2}\\
        =& \frac{1}{\eta^2}((f(z_{k})-f^*)-(f(z_{k+1})-f^*))+\frac{\alpha_k\|z_k-x_0\|_2^2-\alpha_{k+1}\|z_{k+1}-x_0\|_2^2}{\eta^2}
    \end{split}
\end{equation}
where we successively used \eqref{eq:decrease_alpha} in the inequality and introduced $f^*$ as defined in \autoref{sec:notations}. We now turn to \eqref{eq:29} and obtain: 
\begin{equation}
    \begin{split}\label{eq:29_new}
    &\frac{L_{h^{(k)}}}{\eta_{k-1}}\left(\frac{1}{c_{k-1}}-1\right)h^{(k)}(x_{k-1})-\frac{L_{h^{(k)}}}{\eta_k}\left(\frac{1}{c_k}-1\right)h^{(k)}(x_k)\\
    =&\frac{L}{\eta}\left(\frac{1}{c}-1\right)\left[f(x_{k-1})-f(x_{k}) +\frac{\alpha_k}{2}\|x_{k-1}-x_0\|_2^2-\frac{\alpha_{k+1}}{2}\|x_k-x_0\|_2^2  \right] \\
    +&\frac{\alpha_k}{\eta}\left(\frac{1}{c}-1\right)\left[f(x_{k-1})-f(x_{k}) +\frac{\alpha_k}{2}\|x_{k-1}-x_0\|_2^2-\frac{\alpha_{k+1}}{2}\|x_k-x_0\|_2^2  \right]  \\
    +&\frac{L+\alpha_k}{\eta}\left(\frac{1}{c}-1\right)(\alpha_{k+1}-\alpha_{k})\|x_k-x_0\|_2^2.
    \end{split}
\end{equation}
By using \eqref{eq:decrease_alpha} and \eqref{eq:alphakk1diff}, the last term in \eqref{eq:29_new} can be bounded as: 
\begin{align}\label{eq:dec29}
    \frac{L+\alpha_k}{\eta}\left(\frac{1}{c}-1\right)(\alpha_{k+1}-\alpha_{k})\|x_k-x_0\|_2^2&\leq 0.
\end{align}
The second term in \eqref{eq:29_new} can be rewritten as
\begin{equation}
    \begin{split}\label{eq:29_newnew}
        &\frac{\alpha_k}{\eta}\left(\frac{1}{c}-1\right)\left[f(x_{k-1})-f(x_{k}) +\frac{\alpha_k}{2}\|x_{k-1}-x_0\|_2^2-\frac{\alpha_{k+1}}{2}\|x_k-x_0\|_2^2  \right]\\
        =&\frac{1}{\eta}\left(\frac{1}{c}-1\right)\left[\alpha_kf(x_{k-1})-\alpha_{k+1}f(x_{k}) +\frac{\alpha_k^2}{2}\|x_{k-1}-x_0\|_2^2-\frac{\alpha_{k+1}^2}{2}\|x_k-x_0\|_2^2  \right]\\
    +&\frac{1}{\eta}\left(\frac{1}{c}-1\right)(\alpha_{k+1}-\alpha_k) f(x_{k}) +\frac{1}{\eta}\left(\frac{1}{c}-1\right)(\alpha_{k+1}-\alpha_k)  \frac{\alpha_{k+1}}{2}\|x_{k+1}-x_0\|_2^2.
    \end{split}
\end{equation}

By using \eqref{eq:alphakk1diff1} and \eqref{eq:24}, \eqref{eq:29_newnew} becomes
\begin{equation}
    \begin{split}\label{eq:29_newnewnew}
        &\frac{\alpha_k}{\eta}\left(\frac{1}{c}-1\right)\left[f(x_{k-1})-f(x_{k}) +\frac{\alpha_k}{2}\|x_{k-1}-x_0\|_2^2-\frac{\alpha_{k+1}}{2}\|x_k-x_0\|_2^2  \right]\\
        =&\frac{\alpha_k}{\eta}\left(\frac{1}{c}-1\right)\left[(f(x_{k-1})-f^*)-(f(x_{k})-f^*) +\frac{\alpha_k}{2}\|x_{k-1}-x_0\|_2^2\right.\\
        &\left.-\frac{\alpha_{k+1}}{2}\|x_k-x_0\|_2^2  \right]\\
        \leq&\frac{1}{\eta}\left(\frac{1}{c}-1\right)\left[\alpha_k(f(x_{k-1})-f^*)-\alpha_{k+1}(f(x_{k})-f^*) +\frac{\alpha_k^2}{2}\|x_{k-1}-x_0\|_2^2\right.\\
        &\left.-\frac{\alpha_{k+1}^2}{2}\|x_k-x_0\|_2^2  \right]-\frac{1}{\eta}\left(\frac{1}{c}-1\right)\frac{f(x_{k})-f^*}{4\eta (k+2)^{3/2}\sqrt{k+3}}\\
        &\leq \frac{1}{\eta}\left(\frac{1}{c}-1\right)\left[\alpha_k(f(x_{k-1})-f^*)-\alpha_{k+1}(f(x_{k})-f^*) +\frac{\alpha_k^2}{2}\|x_{k-1}-x_0\|_2^2\right.\\
        &\left.-\frac{\alpha_{k+1}^2}{2}\|x_k-x_0\|_2^2  \right],
    \end{split}
\end{equation}
where we used $f(x_k)-f^*>0$ in the last inequality. Finally, we deal with \eqref{eq:30}.
\begin{equation}\label{eq:new_30}
    \begin{split}
        &\frac{1}{2}\frac{L_{h^{(k)}}}{\eta_{k-1}^2 c_{k}^2}\|x_{k}-x_{k-1}\|_2^2-\frac{1}{2}\frac{L_{h^{(k)}}}{\eta_k^2 c_{k+1}^2}\|x_{k+1}-x_k\|_2^2\\
        =&\frac{1}{2}\frac{L}{\eta^2  c^2}(\|x_{k}-x_{k-1}\|_2^2-\|x_{k+1}-x_k\|_2^2)\\
        +&\frac{1}{2\eta^2 c^2}(\alpha_k\|x_{k}-x_{k-1}\|_2^2-\alpha_{k+1}\|x_{k+1}-x_k\|_2^2)\\
    +&\frac{\alpha_{k+1}-\alpha_k}{2\eta^2 c^2}\|x_{k+1}-x_k\|_2^2.
    \end{split}
\end{equation}
By using \eqref{eq:decrease_alpha}, the last term in \eqref{eq:new_30} is upper bounded
\begin{align}\label{eq:dec30}
    \frac{\alpha_{k+1}-\alpha_k}{2\eta^2 c^2}\|x_{k+1}-x_k\|_2^2\leq 0.
\end{align}

Now, by assembling \eqref{eq:new_28}, \eqref{eq:29_new}, \eqref{eq:new_30}, \eqref{eq:29_newnewnew}, \eqref{eq:new_30} and \eqref{eq:dec30}, we have
\begin{equation}
    \begin{split}\label{eq:final_lyap_diff}
        \Gamma_k-\mathbb{E}_{k}[\Gamma_{k+1}]
        &\leq \frac{1}{\eta^2}((f(z_{k})-f^*)-(f(z_{k+1})-f^*))\\
        &+\frac{\alpha_k\|z_k-x_0\|_2^2-\alpha_{k+1}\|z_{k+1}-x_0\|_2^2}{\eta^2}\\
    &+\frac{L}{\eta}\left(\frac{1}{c}-1\right)\left[(f(x_{k-1})-f^*) -(f(x_{k})-f^*)\right.\\
    &\left.+\frac{\alpha_k}{2}\|x_{k-1}-x_0\|_2^2-\frac{\alpha_{k+1}}{2}\|x_k-x_0\|_2^2  \right] \\
&+\frac{1}{\eta}\left(\frac{1}{c}-1\right)\bigg[\alpha_k(f(x_{k-1})-f^*)-\alpha_{k+1}(f(x_{k})-f^*) \\
& +\frac{\alpha_k^2}{2}\|x_{k-1}-x_0\|_2^2-\frac{\alpha_{k+1}^2}{2}\|x_k-x_0\|_2^2  \bigg]\\
    &+\frac{1}{2}\frac{L}{\eta^2  c^2}(\|x_{k}-x_{k-1}\|_2^2-\|x_{k+1}-x_k\|_2^2)\\
    &+\frac{1}{2\eta^2 c^2}(\alpha_k\|x_{k}-x_{k-1}\|_2^2-\alpha_{k+1}\|x_{k+1}-x_k\|_2^2).
    \end{split}
\end{equation}

 \paragraph{Finalizing the convergence bound.} By summing \eqref{eq:finalbeforelyap} for $k=0,\dots,T$, plugging the bound \eqref{eq:final_lyap_diff} on the difference of Lyapunov functions and using the condition \eqref{eq:cond_lambda} on the stepsize parameter $\eta$, we obtain: 
 
  \begin{equation}\label{eq:finalalmost}
\begin{split}
        &\frac{1}{4T}\sum_{k=0}^T\|\nabla f(x_k)\|_2^2+\frac{1}{4T}\sum_{k=0}^T\|\nabla f(z_k)\|_2^2\\
        &\leq \frac{(f(z_{0})-f^*)-(f(z_{T+1})-f^*)}{\eta T}+\frac{\alpha_0\|z_0-x_0\|_2^2}{\eta T}\\
        &+\left(\frac{1}{c}-1\right)\frac{(L+\alpha_0)(f(x_{0})-f^*)-(L+\alpha_T)(f(x_{T})-f^*) }{T} \\
    &
    +\frac{1}{2}\frac{L+\alpha_0}{\eta c^2}(\|x_{0}-x_{-1}\|_2^2)\\
        &+  \left(L\eta +\frac{1}{2T}\sum_{k=0}^T\frac{1}{k+1}\right)\sigma^2+ \left(\left(\frac{1}{\eta}+L \right)\frac{1}{2T} \sum_{k=0}^T\frac{1}{k+1}+\frac{1}{4\eta T }\sum_{k=0}^T\frac{1}{(k+1)^2}\right) R^2.
\end{split}
\end{equation}

By using the inequality $\sum_{k=0}^T \frac{1}{k+1}\leq \log(T+1)$ and setting $z_0=x_0$, $x_{-1}=x_0$ and $\eta=1/\sqrt{T}$ for $T \geq L^2/c^2$, we obtain the aimed result.

\section{Experimental setup}
\label{sec:experimental-setup}

\subsection*{CIFAR10}
Our data augmentation pipeline consisted of random horizontal flipping, then random crop to 32x32, then normalization by centering around (0.5, 0.5, 0.5).
The learning rate schedule normally used for SGD, consisting of a 10-fold decrease at epochs 150 and 225 was found to work well for MDA and Adam. Flat schedules as well as inverse-sqrt schedules did not work as well.

\begin{tabular}{|c|c|}
\hline 
Hyper-parameter & Value\tabularnewline
\hline 
\hline 
Architecture & PreAct ResNet152\tabularnewline
\hline 
Epochs & 300\tabularnewline
\hline 
GPUs & 1xV100\tabularnewline
\hline 
Batch Size per GPU & 128\tabularnewline
\hline 
Decay & 0.0001\tabularnewline
\hline 
Seeds & 10\tabularnewline
\hline 
\end{tabular}

\subsection*{ImageNet}
Data augmentation consisted of the RandomResizedCrop(224) operation in PyTorch, followed by RandomHorizontalFlip then normalization to mean=[0.485, 0.456, 0.406] and std=[0.229, 0.224, 0.225]. The standard schedule for SGD, where the learning rate is decreased 10 fold every 30 epochs, was found to work well for MDA also. No alternate schedule worked well for Adam.

\begin{tabular}{|c|c|}
\hline 
Hyper-parameter & Value\tabularnewline
\hline 
\hline 
Architecture & ResNet50\tabularnewline
\hline 
Epochs & 100\tabularnewline
\hline 
GPUs & 8xV100\tabularnewline
\hline 
Batch size per GPU & 32\tabularnewline
\hline 
Decay & 0.0001\tabularnewline
\hline 
Seeds & 5\tabularnewline
\hline 
\end{tabular}

\subsection*{fastMRI}
For this task, the best learning rate schedule is a flat schedule, with a small number fine-tuning epochs at the end to stabilize. To this end, we decreased the learning rate 10 fold at epoch 40.

\begin{tabular}{|c|c|}
\hline 
Hyper-parameter & Value\tabularnewline
\hline 
\hline 
Architecture & 12 layer VarNet 2\tabularnewline
\hline 
Epochs & 50\tabularnewline
\hline 
GPUs & 8xV100\tabularnewline
\hline 
Batch size per GPU & 1\tabularnewline
\hline 
Decay & 0.0\tabularnewline
\hline 
Acceleration factor & 4\tabularnewline
\hline 
Low frequency lines & 16\tabularnewline
\hline 
Mask type & Offset-1\tabularnewline
\hline 
Seeds & 5\tabularnewline
\hline 
\end{tabular}

\subsection*{IWSLT14}
Our implementation used FairSeq defaults except for the parameters listed below. For the learning rate schedule, ADAM used the inverse-sqrt, whereas we found that either fixed learning rate schedules or polynomial decay schedules worked best, with a decay coefficient of 1.0003 starting at step 10,000.

\begin{tabular}{|c|c|}
\hline 
Hyper-parameter & Value\tabularnewline
\hline 
\hline 
Architecture & transformer\_iwslt\_de\_en\tabularnewline
\hline 
Epochs & 55\tabularnewline
\hline 
GPUs & 1xV100\tabularnewline
\hline 
Max tokens per batch & 4096\tabularnewline
\hline 
Warmup steps & 4000\tabularnewline
\hline 
Decay & 0.0001\tabularnewline
\hline 
Dropout & 0.3\tabularnewline
\hline 
Label smoothing & 0.1\tabularnewline
\hline 
Share decoder/input/output embed & True\tabularnewline
\hline 
Float16 & True\tabularnewline
\hline 
Update Frequency & 8\tabularnewline
\hline 
Seeds & 20\tabularnewline
\hline 
\end{tabular}

\subsection*{RoBERTa}
Our hyper-parameters follow the released \href{https://github.com/pytorch/fairseq/blob/master/examples/roberta/README.pretraining.md}{documentation} closely. We uesd the same hyper-parameter schedule for ADAM and SGD, with different learning rates chosen by a grid search.

\begin{tabular}{|c|c|}
\hline 
Hyper-parameter & Value\tabularnewline
\hline 
\hline 
Architecture & roberta\_base\tabularnewline
\hline 
Task & masked\_lm\tabularnewline
\hline 
Max updates & 20,000\tabularnewline
\hline 
GPUs & 8xV100\tabularnewline
\hline 
Max tokens per batch & 4096\tabularnewline
\hline 
Decay & 0.01 (ADAM) / 0.0 (MDA) \tabularnewline
\hline 
Dropout & 0.1\tabularnewline
\hline 
Attention dropout & 0.1\tabularnewline
\hline 
Tokens per sample & 512\tabularnewline
\hline 
Warmup & 10,000\tabularnewline
\hline 
Sample break mode & complete\tabularnewline
\hline 
Skip invalid size inputs valid test & True\tabularnewline
\hline 
LR scheduler & polynomial\_decay\tabularnewline
\hline 
Max sentences & 16\tabularnewline
\hline 
Update frequency & 16\tabularnewline
\hline 
tokens per sample & 512\tabularnewline
\hline 
Seeds & 1\tabularnewline
\hline 
\end{tabular}

\subsection*{Wikitext}
Our implementation used FairSeq defaults except for the parameters listed below. The learning rate schedule that gave the best results for MDA consisted of a  polynomial decay starting at step 18240 with a factor 1.00001.

\begin{tabular}{|c|c|}
\hline 
Hyper-parameter & Value\tabularnewline
\hline 
\hline 
Architecture & transformer\_lm\tabularnewline
\hline 
Task & language\_modeling\tabularnewline
\hline 
Epochs & 46\tabularnewline
\hline 
Max updates & 50,000\tabularnewline
\hline 
GPUs & 1xV100\tabularnewline
\hline 
Max tokens per batch & 4096\tabularnewline
\hline 
Decay & 0.01\tabularnewline
\hline 
Dropout & 0.1\tabularnewline
\hline 
Tokens per sample & 512\tabularnewline
\hline 
Warmup & 4000\tabularnewline
\hline 
Sample break mode & None\tabularnewline
\hline 
Share decoder/input/output embed & True\tabularnewline
\hline 
Float16 & True\tabularnewline
\hline 
Update Frequency & 16\tabularnewline
\hline 
Seeds & 20\tabularnewline
\hline 
\end{tabular}

\section{Further experiments}\label{sec:further_exps}

\begin{figure}[t]
       \begin{minipage}{.5\textwidth}
	\centering
        \includegraphics[width=.85\linewidth]{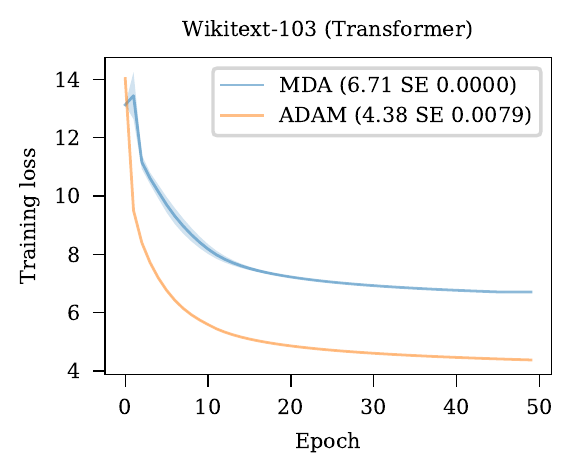}
     \end{minipage}%
    \begin{minipage}{0.45\textwidth}
        \resizebox{1\textwidth}{!}{
         \begin{tabular}{|l|c|c|c|}
  \hline
  & SGD & Adam & MDA \\
  \hline
  Perplexity & $78.5$ &  $33.05\pm0.16$ & $31.54\pm0.11$ \\
  \hline
\end{tabular}}
     \end{minipage}
     \caption{Training loss performance on the wikitext-103 language modeling task (left) and the test performance (right). Some over-fitting is observed for ADAM.}
\end{figure}

\subsection{Language modeling}
We use Wikitext-103 \citep{merity2016pointer} dataset which contains 100M tokens. 
Following the setup in \cite{ott2019fairseq}, we train a six-layer tensorized transformer and report the perplexity (PPL)
on the test set. \autoref{fig:nlp} (c) reports the training loss and (d) the perplexity score on the test set of SGD (as reported in \citep{yao2020adahessian}), MDA and Adam on Wikitext-103.
We note that the MDA reaches a significantly worse training loss than Adam (2.33) in this case. Yet, it achieves a better perplexity (1.51) on the test set. As with CIFAR-10, we speculate this is due to the additional decaying regularization that is a key part of the MDA algorithm.

\end{document}